\documentclass[11pt]{article}
\usepackage[final]{acl}
\usepackage{times}
\usepackage{latexsym}
\usepackage[T1]{fontenc}
\usepackage[utf8]{inputenc}
\usepackage{microtype}
\usepackage{amsmath,amssymb,amsthm}
\usepackage{booktabs}
\usepackage{graphicx}
\newcommand{\Rord}{R_{\mathrm{ord}}}
\newcommand{\bflow}{b_{\mathrm{flow}}}
\newcommand{\IDCG}{\operatorname{IDCG}}
\newcommand{\kapeps}{\kappa_{\varepsilon}}
\newtheorem{theorem}{Theorem}
\newtheorem{proposition}{Proposition}
\title{Exact Risk Certificates for Pairwise-Judged Reranking}
\author{
  \textbf{Xinyu Wang}$^{1,*}$ \quad \textbf{Jingrui Tian}$^{1,*}$ \quad \textbf{Sicheng Lyu}$^{1}$ \quad \textbf{Yu Gu}$^{1}$ \\
  \textbf{Ziyang Song}$^{2}$ \quad\textbf{Xiao-Wen Chang}$^{1}$\quad   \textbf{Lei Ding}$^{3,\dagger}$ \\
  { $^{1}$McGill University \quad $^{2}$Ohio University \quad $^{3}$University of Manitoba} \\
  {\textrm{\{xinyu.wang5, jingrui.tian\}@mail.mcgill.ca}} \\
  {\textrm{lei.ding@umanitoba.ca}}
}
\begin{document}
\maketitle
\begingroup
\renewcommand{\thefootnote}{\fnsymbol{footnote}}
\footnotetext[1]{Equal contribution.}
\footnotetext[2]{Corresponding author.}
\endgroup
\begin{abstract}
Pairwise relevance judgments constrain both the loss numerator and the ideal-ranking denominator of nDCG. We characterize the exact worst-case loss of replacing a ranking with a fixed proposal. Binary gains attain this loss and provide an inspectable witness. Threshold verification reduces to at most $K+1$ maximum-closure problems. Our exposure-flow ablation isolates denominator coupling: on shared judge trajectories, exact verification saves 2.34 comparisons per TREC-DL development query and 1.98 per confirmation query at a budget of 32. Savings transfer to BEIR and a second proposer. Predictive gates achieve higher deployment gains, while certificates expose the worst gains still permitted by accepted comparisons.
Our code is available at \url{https://github.com/W4ng77/certified-reranking-code}.
\end{abstract}

\section{Introduction}
LLM rerankers improve retrieval through relevance comparisons \citep{qin-etal-2024-large,luo-etal-2024-prp}. Average gains can coexist with query-level losses \citep{dincer2014hypothesis,wang2012robust}. Before installing a fixed proposal, a system can buy comparisons to certify that its nDCG loss stays within tolerance.

Comparisons constrain both the loss numerator and its ideal-ranking denominator. If a document losing exposure is relevant, any document judged at least as relevant must also contribute to the ideal ranking. Our exposure-flow ablation discards closure-induced denominator growth. It can therefore continue asking after the comparison evidence already suffices.

\begin{figure*}[t]
\centering
\includegraphics[width=\textwidth]{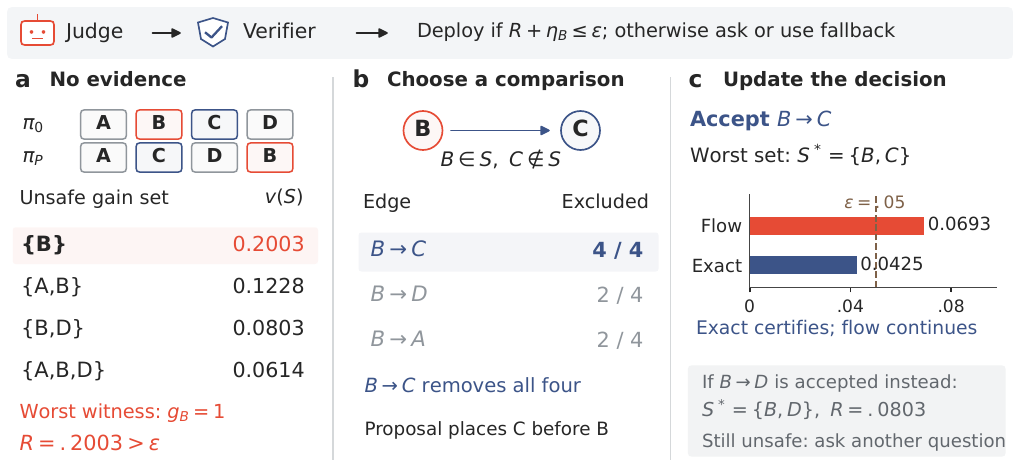}
\caption{\textbf{A witness explains the next question.} For $K=4$, $\varepsilon=.05$, and $w\approx(1,.6309,.5,.4307)$, a set assigns gain one to its members and zero elsewhere, with $v(S)=a(S)/Z_{\min(|S|,K)}$. (a) Four sets remain unsafe without evidence. (b) $B\to C$ excludes all four; the other shown assertions exclude two. (c) Accepting it yields worst set $\{B,C\}$: exact certifies while flow continues. The gray branch shows the next unsafe witness after $B\to D$. $B\to A$ is illustrative: $a_A=0$, so it is outside the purchased support pairs.}
\label{fig:coupling}
\end{figure*}

We contribute an exact comparison-conditioned nDCG risk with inspectable binary witnesses, a threshold test using at most $K+1$ maximum-closure solves, and a shared-trajectory ablation measuring the evidence value of denominator coupling. Our top-10 experiments enumerate unsafe sets for question selection.

\section{Comparison-Conditioned Risk}
\label{sec:exact}
Let $\pi_0$ be the incumbent and $\pi_P$ a fixed proposal. Their top-$K$ union is $C$, with $N=|C|\le2K$. For discount $w_j=1/\log_2(j+1)$, exposure $x_\pi(d)$ is $w_j$ at position $j\le K$ and zero otherwise. Set
\begin{equation}
 a_d=x_{\pi_0}(d)-x_{\pi_P}(d),\quad
 Z_m=\sum_{j=1}^{m}w_j.
\end{equation}
Both rankings have total exposure $Z_K$, so $a(C)=\sum_{d\in C}a_d=0$. For non-negative gains $g$ on $C$, let $\IDCG_K(g)=\sum_{j\le K}w_jg_{[j]}$, with gains sorted decreasingly and padded with zeros. Ranking utility is nDCG: $U(\pi;g)=\sum_d x_\pi(d)g_d/\IDCG_K(g)$.

An accepted edge $i\to j$ asserts $g_j\ge g_i$. With edges $E$, define
\begin{align}
\mathcal G(E)&=\{g\ge0:g_j\ge g_i\ \forall(i,j)\in E,\nonumber\\[-2pt]
&\hspace{53pt}\IDCG_K(g)>0\},\\
\Rord(E)&=\sup_{g\in\mathcal G(E)}\frac{a^\top g}{\IDCG_K(g)}.
\label{eq:risk}
\end{align}
The subscript ``ord'' denotes ordinal evidence: weak comparisons without absolute grades or relevance counts. We retain documents with $a_d=0$ because their comparisons can constrain the denominator. When all accepted assertions are correct, the proposal loss $r_P=U(\pi_0)-U(\pi_P)$ is at most $\Rord(E)$. The bound also covers full-corpus nDCG: gains outside $C$ can only increase the denominator, and the supremum is non-negative.

A non-empty $S\subseteq C$ is \emph{$E$-closed} if $i\in S,(i,j)\in E$ implies $j\in S$. Let $\mathcal F(E)$ contain these sets, and define
\begin{equation}
 v(S)=\frac{a(S)}{Z_{\min(|S|,K)}}.
\end{equation}
\begin{theorem}[Exact risk]
\label{thm:exact}
The risk is $\Rord(E)=\max_{S\in\mathcal F(E)}v(S)$, attained by the binary gains $g_d=\mathbf1[d\in S]$ of a maximizing set.
\end{theorem}
\paragraph{Proof idea.}
Admissible gains decompose into closed binary threshold sets. Their loss ratio is a weighted average of the binary losses, so a threshold set attains the maximum. Appendix~\ref{app:proof} proves the graded-gain and full-corpus result.

Since $C$ is closed and $a(C)=0$, the maximum is non-negative. If $\Rord(E)>\varepsilon$, its maximizing set supplies gains consistent with every comparison yet exceeding the loss tolerance. A sound rule for this information model therefore cannot certify that proposal. Additional absolute relevance information can change the feasible set.

Certification can finish before the relevance order is known. Unresolved comparisons do not matter once every remaining gain assignment stays within tolerance. When certification fails, the witness identifies what is missing: any comparison contradicting that assignment removes it. The same step can then remove other unsafe witnesses.

\paragraph{Why flow can be loose.}
For $E_{\rm sup}\subseteq I\times J$, from exposure losers $I=\{i:a_i>0\}$ to exposure gainers $J=\{j:a_j<0\}$, let $F_{\rm sup}$ be maximum flow with capacities $a_i$ and $-a_j$, and $L=\sum_{i\in I}a_i$. Let $s_{(r)}$ be the positive $a_i$ in descending order, padded with zeros. The no-evidence risk is
\begin{equation}
 R_0=\max_{1\le m\le K}\frac{\sum_{r\le m}s_{(r)}}{Z_m},
\end{equation}
For $E=E_{\rm sup}$, maximum closure \citep{picard1976closure} gives $\max_{S\in\mathcal F(E)}a(S)=L-F_{\rm sup}$. Hence
\begin{equation}
 \frac{L-F_{\rm sup}}{Z_K}\le\Rord(E)\le
 \underbrace{\min\{R_0,L-F_{\rm sup}\}}_{\bflow(E)}.
\label{eq:sandwich}
\end{equation}
The upper bound uses $Z_{\min(|S|,K)}\ge1$, discarding closure-induced denominator growth. The lower bound is specific to the support graph; it need not hold after adding other edges.

Figure~\ref{fig:coupling} uses $K=4$ and a proposal moving $B$ from second to fourth. Accepting $B\to C$ excludes all four unsafe gain sets. The new maximizing set $\{B,C\}$ gives
\begin{equation}
 \Rord(E)=\frac{w_3-w_4}{w_1+w_2}=.0425.
\end{equation}
At $\varepsilon=.05$, exact certifies while flow remains at .0693. Accepting $B\to D$ instead leaves an unsafe witness with risk .0803.

\paragraph{Computation.}
Let $M_E(q)$ be maximum closed-set weight, including the empty set. The concave denominator $f(m)=Z_{\min(m,K)}$ is the lower envelope $\min_\ell(c_\ell+d_\ell m)$, where \mbox{$d_\ell=w_{\ell+1}$} and \mbox{$c_\ell=Z_\ell-\ell d_\ell$} for $0\le\ell\le K$, with $Z_0=w_{K+1}=0$. Each line touches $f$ and lies above it. For $\lambda\ge0$,
\begin{align}
 \Phi(\lambda)&=\max_{S\text{ closed}}[a(S)-\lambda f(|S|)]\nonumber\\
 &=\max_\ell\{M_E(a-\lambda d_\ell\mathbf1)-\lambda c_\ell\}.
\label{eq:closure-test}
\end{align}
Each inner problem is maximum closure. Thus $K+1$ min-cuts test $\Rord(E)\le\lambda$ without subset enumeration. Appendix~\ref{app:implementation} gives the construction. The experimental cover planner retains enumeration, keeping at most 20,000 unsafe sets for question selection; risk evaluation uses all sets.

\paragraph{From counterexamples to questions.}
An edge $i\to j$ excludes an unsafe set $S$ when $i\in S$ and $j\notin S$. We greedily cover unsafe sets with pairs and purchase one using a fixed proposer-score prior. The experiment retains only accepted planned edges; other answers still cost budget. On budget exhaustion, it uses a verified fallback, which can always keep the incumbent. Appendix~\ref{app:protocol} gives the planner and calibration details.

\section{Evaluation}
\label{sec:eval}
\paragraph{Protocol.}
BM25 supplies the incumbent, Qwen3-Reranker-0.6B the proposal, and Qwen3-32B the judge \citep{qwen3}. We use $K=10$, $\varepsilon=.05$, and primary comparison budget $B=32$. Development uses TREC-DL 2020--2022 (61 calibration / 122 test queries); confirmation uses TREC-DL 2023 (27/55). Transfer uses three BEIR corpora \citep{thakur2021beir}; BGE-reranker-v2-m3 supplies a second proposal on confirmation.

Each paid comparison uses two prompt orientations. Mean support at .50 accepts the planned weak assertion; tie probability supports that direction. Rejections cost budget and add no edge. We replay exact and flow gates on identical accepted evidence and calibrate each separately for imperfect judgments using a trajectory-wise split-conformal correction \citep{vovk2005algorithmic}. For a gate's bound $b_t$ at prefix $t$, the calibration score is
\begin{equation}
 S_B(q)=\max_{0\le t\le B}[r_P(q)-b_t(q)]_+.
\end{equation}
We use the finite-sample corrected $(1-\alpha)$ quantile $\eta_B$, with $\alpha=.10$. A release installs the proposal at the first prefix with $b_t+\eta_B\le\varepsilon$; otherwise it uses a verified fallback. Under exchangeability and a frozen policy, coverage holds across queries and all stopping prefixes. Development selects the acceptance rule and has no split-conformal guarantee; confirmation freezes it (Appendix~\ref{app:protocol}).

\begin{figure*}[t]
\centering
\includegraphics[width=\textwidth]{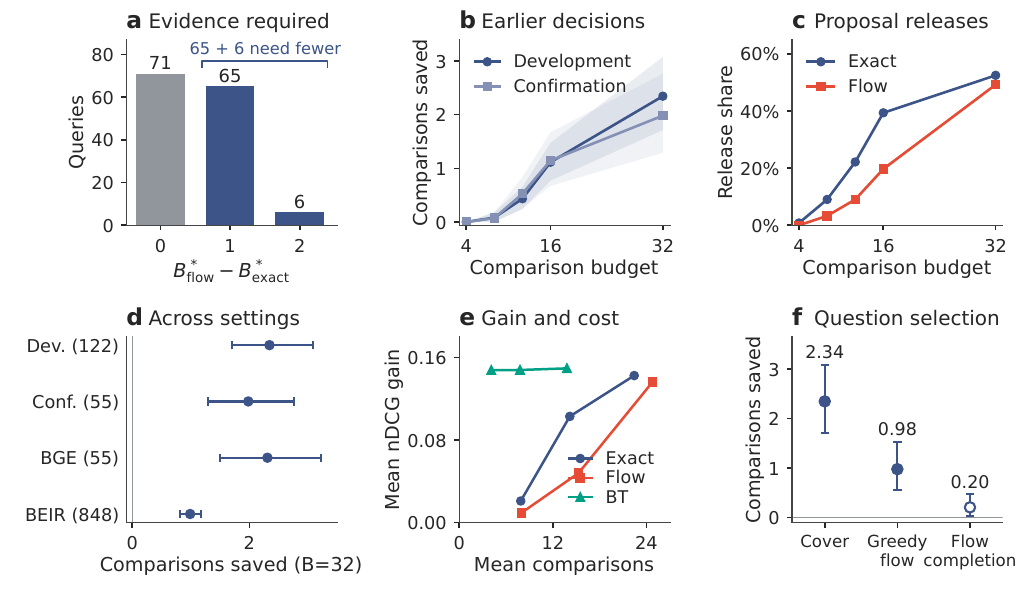}
\caption{\textbf{Exact verification reduces evidence cost.} (a) Optimal hindsight savings on 142 cases with correct comparisons. (b) Development and confirmation stopping. (c) Development release shares. (d) Savings at $B=32$ across four settings. (e) Development gain and cost at budgets 8, 16, and 32. (f) Savings depend on acquisition; the hollow point has Holm-adjusted $p\ge.05$. Shading and error bars show paired 95\% bootstrap intervals.}
\label{fig:experiments}
\end{figure*}

\paragraph{Less waiting on the same evidence.}
Cost is $\min(T,B)$, where $T$ is the first certifying prefix. On development, exact saves 2.34 comparisons per query, with paired 95\% bootstrap CI $[1.71,3.07]$, against flow's 24.77 (9.46\%; Fig.~\ref{fig:experiments}b). Among the 118 queries with the same final decision under both gates, it saves 1.88, CI $[1.42,2.40]$. Savings include unchanged outcomes.

Figure~\ref{fig:experiments}c shows where the saving arises. At $B=16$, exact releases 48/122 proposals, versus 24/122 for flow; at $B=32$, the counts approach 64/122 and 60/122. The main gain is earlier decisions, with similar final release counts. Their corrections are zero, so Eq.~\ref{eq:sandwich} determines stopping order. A zero empirical quantile does not establish judge accuracy; different corrections can reverse the order.

Savings extend to confirmation and transfer (Fig.~\ref{fig:experiments}d): 1.98 comparisons on confirmation, 2.31 with the second proposer, and .99 on pooled BEIR. All four cost tests survive Holm correction. Each BEIR corpus contributes savings of .533 on SciFact, .986 on NFCorpus, and 1.199 on FiQA. Gain differences are less consistent: development gains .0550 nDCG at $B=16$, but the confirmation and second-proposer gain tests do not survive correction. NFCorpus and FiQA also show more observed violations (Appendix~\ref{app:results}).

Acquisition affects its magnitude. Development savings at $B=32$ are 2.34 under cover, .98 under greedy flow, and .20 under flow completion; the last contrast is not significant after correction. Completion already acquires evidence tailored to the looser bound. These measurements establish comparison savings, without measuring end-to-end latency. Verifier construction averages 5.27\,ms and an accepted-edge update 1.20\,ms; cover planning averages 154\,ms per call.

Cover excludes unsafe normalized gain sets, allowing an edge to change both numerator and denominator. Flow completion targets the exposure-transfer deficit. This helps explain the savings.

\paragraph{Evidence sufficiency and deployment scope.}
With complete correct support comparisons, exact certifies 146 of 183 development proposals, versus 142 for flow. Among the 142 common cases, it needs fewer edges for 71 and more for none (Fig.~\ref{fig:experiments}a). These hindsight optima count correct evidence, excluding unsuccessful purchases. Even the complete weak order leaves 22 actually safe proposals uncertifiable: their ordinal evidence still admits unsafe gain magnitudes. See Appendix~\ref{app:evidence} for the optimization.

A separately calibrated Bradley--Terry (BT) predictor \citep{bradley1952rank} achieves greater deployment gain (Fig.~\ref{fig:experiments}e). With greedy-flow acquisition on development, BT gains .1593 nDCG using 13.4 comparisons; exact gains .1363 using 23.1. With cover on confirmation, BT gains .1291 versus .0964 for exact. BT uses probabilities from the same paid judgments and a score prior. The certificate instead exposes the worst gain assignment allowed by accepted assertions. Full baselines appear in Appendix~\ref{app:results}. Always deploying exceeds $\varepsilon$ on 8/122 development and 1/55 confirmation queries. Full development has nine tolerated losses and fourteen over-budget proposals (Appendix~\ref{app:results}).

A search team can require a checkable loss bound before publishing a reranking update. Reviewers can verify it from accepted comparisons. A failed certificate identifies permitted gains and guides the next comparison.

\section{Related Work}
Minimal test collections use absolute relevance judgments to determine the sign of a system's AP difference \citep{carterette2006minimal}; preference-based evaluation selects pairs to reduce assessment effort \citep{carterette2008here}. MED bounds ranking differences over unknown relevance \citep{tan2015med,moffat2018med}. Our verifier bounds a fixed proposal's normalized loss using weak inequalities. The level-set argument specializes Lov\'{a}sz-extension ratio principles \citep{bach2013submodular,hein2011beyond}; the $K+1$-cut construction uses maximum closure \citep{picard1976closure,gallo1989parametric}. LLM assessors can predict relevance preferences, with reliability depending on judgment design \citep{thomas2024llm,faggioli2023perspectives}. Risk-sensitive optimization trains robust rankings \citep{wang2012robust}; we verify proposed rankings.

\section{Conclusion}
Weak comparisons admit an exact nDCG loss certificate and a polynomial threshold test. Denominator coupling reduces evidence cost on shared judge trajectories. Predictive gates achieve higher gains. Binary witnesses expose the remaining uncertainty and give a concrete target for the next comparison.
\label{main:end}
\clearpage

\section*{Limitations}
The guarantee is relative to weak comparisons, unknown non-negative gains, and a fixed proposal. Specified grade alphabets, absolute judgments, or gain bounds can restrict this model. Merely requiring a positive gain in $C$ adds no information: admissibility already requires positive ideal DCG. Correctness of accepted assertions is required for the direct certificate; the calibrated deployment guarantee instead requires exchangeability and a frozen pipeline. Neither statement provides per-query protection from arbitrary judge errors. Appendix~\ref{app:errors} characterizes a bounded-error extension.

Threshold verification is polynomial, but the experimental cover planner enumerates unsafe subsets. That planner is exponential and practical here at $K=10$. Both implementations use floating point without interval certificates; the closure algorithm has not been evaluated on large retrieval panels. Comparison counts exclude model batching, token lengths, and acquisition overhead, so they do not establish latency or energy savings.

The cover-trajectory analysis is exploratory. Development reuse limits its calibration interpretation, and the 55-query confirmation test set contains only one genuinely over-budget proposal. Pooled BEIR calibration offers a marginal guarantee over the pool, without separate coverage guarantees for each corpus. The results support comparison savings; they do not establish better deployment safety than predictors.

The full-table diagnostic uses a quantized judge on a separate 183-query panel. It measures assertion errors and decision tolerance without a calibration/test split. Official qrels supply the evaluation target; incompleteness matters particularly for unfiltered BEIR candidates. These restrictions limit conclusions about human relevance, other judges, and deployment distributions.

\section*{Ethical Considerations}
Relevance certification can affect which documents users encounter. The loss criterion is defined by benchmark relevance and does not measure factuality, harmful content, fairness, or source quality. LLM comparisons may reproduce judge biases, and small nDCG loss does not imply protection on these other dimensions. We use existing retrieval benchmarks without collecting new participant data or deploying the system to users. Any application should interpret the certificate with its evidence assumptions and evaluate consequences beyond retrieval effectiveness.

\appendix

\section{Proof of the Exact Risk Characterization}
\label{app:proof}
\begin{proof}[Proof of Theorem~\ref{thm:exact}]
For an admissible gain vector $g$, define $S_t=\{d\in C:g_d\ge t\}$ for $t>0$. If $i\in S_t$ and $(i,j)\in E$, then $g_j\ge g_i\ge t$, so $j\in S_t$. Thus every non-empty threshold set belongs to $\mathcal F(E)$.

The identity $g_d=\int_0^\infty\mathbf1[d\in S_t]dt$ yields
\begin{equation}
a^\top g=\int_0^\infty a(S_t)\,dt.
\end{equation}
At level $t$, precisely $|S_t|$ gains reach that level. Summing the discounts of the largest $K$ gains therefore gives
\begin{equation}
\IDCG_K(g)=\int_0^\infty Z_{\min(|S_t|,K)}\,dt,
\end{equation}
where the contribution of the empty set is zero. Let $r=\max_{S\in\mathcal F(E)}v(S)$. For every non-empty threshold set, $a(S_t)\le rZ_{\min(|S_t|,K)}$. Integrating and dividing by the positive ideal DCG proves $a^\top g/\IDCG_K(g)\le r$.

There are finitely many closed sets and $C$ is one, so a maximizer $S^*$ exists. Set $g_d=\mathbf1[d\in S^*]$. Closure ensures that every comparison holds. This vector has numerator $a(S^*)$ and denominator $Z_{\min(|S^*|,K)}>0$, attaining $r$. Taking the supremum proves the equality.
\end{proof}

\paragraph{Full-corpus evaluation.}
The set $C$ is closed and $a(C)=0$, so $r\ge0$. Gains outside $C$ do not affect the loss numerator and can only increase ideal DCG. If the numerator is non-negative, full-corpus normalization can therefore only lower the loss ratio. If it is negative, the ratio is already below $r$. When all gains inside $C$ are zero, the numerator is zero. Hence the conditional upper bound remains valid for full-corpus nDCG whenever its denominator is positive.

\paragraph{Gain alphabets and witnesses.}
The upper-bound argument covers all non-negative graded gains. Any finite allowed gain alphabet containing zero and a common positive value also attains the same maximum: scaling the binary witness by that value preserves its ratio. If $r>\varepsilon$, this witness is a feasible counterexample, so a rule valid for every gain vector in the stated information model cannot certify the proposal. This conclusion does not apply after additional absolute relevance information restricts that model.

\section{Implementation and the Flow Bound}
\label{app:implementation}
The implementation retains all documents in the top-$K$ union, including those with zero exposure difference. For each subset, a recursion stores its weight sum and cardinality. The no-evidence array contains $v(S)$ for every non-empty subset; an accepted edge $i\to j$ invalidates precisely those subsets with $i\in S$ and $j\notin S$. The maximum over remaining entries is the exact risk. Repeating an already accepted edge leaves the risk unchanged.

For the support graph, the maximum-closure reduction uses source capacities $a_i$ for positive weights, sink capacities $-a_j$ for negative weights, and implication capacities greater than $L$. A finite cut associated with a closed set $S$ has capacity $L-a(S)$ \citep{picard1976closure}. This establishes $\max_S a(S)=L-F_{\rm sup}$. Here $F_{\rm sup}$ uses only accepted edges in $I\times J$. Its denominator lies between 1 and $Z_K$. In addition, removing gain constraints cannot lower the supremum, giving $\Rord(E)\le R_0$. With no edges, an optimal positive-loss subset contains only the largest positive weights, of which there are at most $K$.

The self-check suite reproduces the worked example and checks the no-evidence identity, monotonicity after adding edges, the flow sandwich, and maximum closure on 400 random instances. Its worst sandwich slack is $2.2\times10^{-16}$. The randomized checks exercise implementation consistency; the proof of Theorem~\ref{thm:exact} establishes the mathematical result. The full non-empty set $C$ remains feasible, ensuring a zero lower bound even when every other surviving set has negative weight.

\paragraph{Polynomial threshold verification.}
Set $Z_0=0$ and $w_{K+1}=0$. For $\ell=0,\ldots,K$, define $d_\ell=w_{\ell+1}$ and $c_\ell=Z_\ell-\ell d_\ell$. Since the increments of $f(m)=Z_{\min(m,K)}$ decrease, every line $c_\ell+d_\ell m$ lies above $f$ and touches it at $m=\ell$; the line at $\ell=K$ equals $Z_K$ for all $m\ge K$. Consequently their lower envelope equals $f$ at every integer cardinality.

\begin{proposition}[Threshold oracle]
For any accepted comparison graph $E$ and $\lambda\ge0$, Eq.~\ref{eq:closure-test} holds when the empty set is included. Its value is zero if $\Rord(E)\le\lambda$, and positive otherwise. A positive maximizing set is an unsafe binary witness. Each test uses at most $K+1$ minimum cuts on $N+2$ vertices and $O(N+|E|)$ arcs.
\end{proposition}
\begin{proof}
For a fixed $S$, subtracting the lower envelope from $a(S)$ gives
\[
a(S)-\lambda f(|S|)=\max_\ell\{a(S)-\lambda(c_\ell+d_\ell|S|)\}.
\]
The two finite maxima over $S$ and $\ell$ commute, proving the identity. The empty set has objective zero; every non-empty set has $f(|S|)>0$, so positivity is equivalent to a ratio exceeding $\lambda$. For each line, use node weights $a_i-\lambda d_\ell$. Connect positive weights to the source, negative weights to the sink, and each accepted implication $i\to j$ with capacity exceeding the total positive weight. Every minimum cut respects all implications and returns a maximum-weight closed set. Subtract $\lambda c_\ell$ and take the largest value.
\end{proof}

This construction applies to arbitrary comparison graphs, including cycles and edges through zero-exposure documents. The original flow relaxation only uses $E_{\rm sup}$; the new cuts shift every node weight and retain all accepted edges. Thus they solve the exact threshold problem rather than repeating the exposure-transfer relaxation. Closure methods have a long history \citep{picard1976closure,gallo1989parametric}; the affine envelope specializes them to this denominator.

With $H=\sum_i[a_i]_+$, risk lies in $[0,H]$. Bisection gives additive precision $\delta$ using $O((K+1)\log(1+H/\delta))$ cuts. A fixed deployment threshold needs only one oracle call. The reference implementation uses the same floating-point flow routine as the support certificate.

\paragraph{Finite exact-value iteration.}
Start with $\lambda=0$. If $\Phi(\lambda)>0$, let $S$ be its witness and replace $\lambda$ by $v(S)$. The positive residual implies $v(S)>\lambda$, while Theorem~\ref{thm:exact} gives $v(S)\le\Rord(E)$. Each update therefore selects a strictly larger value from a finite family. Iteration terminates in exact arithmetic, and $\Phi(\lambda)=0$ at termination gives $\lambda=\Rord(E)$. This proves finite exact recovery without a polynomial iteration bound. The implementation stops at a residual tolerance; because non-empty denominators are at least one, a residual at most $\tau$ implies $\Rord(E)\le\lambda+\tau$ in exact arithmetic. Floating-point roundoff remains a separate issue. For each fixed line, node weights decrease with $\lambda$, allowing parametric maximum-flow methods to trace the closure breakpoints \citep{gallo1989parametric}; the supplied implementation uses ordinary cuts.

The new oracle was compared with exhaustive closed-set enumeration on 300 random instances with $2\le N\le11$, arbitrary directed edges, and varying $K$. All 1,200 threshold checks agreed within $10^{-9}$; maximum objective error was $9.99\times10^{-16}$. Bisection intervals also covered the enumerated risk. A 100-document, $K=50$ strongly connected test checked the known zero-risk solution. The finite witness-ratio iteration also matched enumeration on these random instances.

The oracle and finite witness-ratio iteration were also checked on all 183 saved real ranking instances, with $11\le N\le20$, $K=10$, and empty evidence. Every objective and risk matched enumeration within $10^{-9}$. These checks exercise real exposure vectors; without implication arcs they do not measure performance on accepted-evidence trajectories. The supplied timing records remain supplementary, and the archived deployment profiles describe the original enumeration implementation.

\paragraph{Planning beyond enumeration.}
An online planner could call the threshold oracle to collect unsafe witnesses and greedily select comparisons that exclude them, adding witnesses as decisions remain unresolved. This avoids materializing the complete unsafe-set family, following the witness-generation principle in Appendix~\ref{app:evidence}. It would change the acquisition policy and require fresh calibration. The reported cover trajectories use exhaustive planning; the proposed witness-sampling planner has not been evaluated.

\begin{table*}[t]
\centering\small
\begin{tabular}{lrrrr}
\toprule
Quantity & mean & median & p95 & max \\
\midrule
verifier construction (ms) & 5.27 & 2.35 & 25.73 & 28.40 \\
per accepted edge, risk update (ms) & 1.20 & 0.50 & 5.58 & 7.07 \\
cover planner, one call (ms) & 154.01 & 176.28 & 213.29 & 224.27 \\
verifier peak memory (MiB) & 11.05 & 5.75 & 46.00 & 46.00 \\
\bottomrule
\end{tabular}

\caption{Recorded computation profile at $K=10$. These component measurements exclude LLM inference and do not establish end-to-end speedups.}
\end{table*}

\section{Acquisition and Calibration}
\label{app:protocol}
\paragraph{Acquisition.}
Counterexample cover greedily covers unsafe sets with unpurchased permitted pairs. It buys the planned pair with the highest frozen acceptance prior,
\[
\sigma\!\left((h_j-h_i)/\max\{\operatorname{IQR}(h),10^{-6}\}\right),
\]
where $h$ denotes proposer scores, $\operatorname{IQR}$ their interquartile range, and $\sigma$ is logistic. Planning deduplicates unsafe sets and retains at most the 20,000 most dangerous ones. Risk evaluation uses the full subset family; planning truncation does not change the certificate. Once the uncorrected risk reaches $\varepsilon$, planning targets zero to acquire calibration margin. Rejections consume budget but grant no edge. We also evaluate greedy flow, minimum-cardinality flow completion, and a static order weighted by capacities.

\paragraph{Common calibration wrapper.}
For a fixed acquisition policy and budget $B$, let $b_t(q)$ be a bound or predicted loss after $t$ purchases. Use the main-text score $S_B$: signed loss $r_P=U(\pi_0)-U(\pi_P)$ is positive for degradation and negative for improvement; $[\cdot]_+$ clips negative residuals to zero.
With $n$ calibration queries, let $\eta_B$ be the $\lceil(n+1)(1-\alpha)\rceil$-th smallest calibration score, or infinity when this rank exceeds $n$. The system installs the proposal at the first $t\le B$ with $b_t+\eta_B\le\varepsilon$. Otherwise it uses a fallback verified to lose at most $\varepsilon$; keeping the incumbent always qualifies. Split-conformal coverage \citep{vovk2005algorithmic} implies, for exchangeable calibration and test scores under this frozen pipeline,
\begin{equation}
\Pr\{[U(\pi_0)-U(\pi_{\mathrm{out}})]_+>\varepsilon\}\le\alpha.
\end{equation}
On $S_B\le\eta_B$, every potential stopping prefix is covered, so adaptive stopping does not invalidate this conclusion. Coverage is marginal over queries. Related risk-control approaches include \citet{angelopoulos2024conformal,angelopoulos2025ltt,xu2025twostage}.

Each gate is calibrated separately, including BT and a score-only predictor. A flow gate's correction need not cover the tighter exact bound. Under an identical correction, Eq.~\ref{eq:sandwich} makes exact stopping no later; with different corrections, stopping order is empirical. The exact and flow corrections are zero across all reported gate, acquisition-policy, and budget configurations, so the observed ordering follows the bound. At $n=27$ and $\alpha=.10$, the corrected quantile is the 26th score. A zero value means that this empirical order statistic is zero; it neither measures assertion accuracy nor guarantees conditional coverage. Smaller $\alpha$ or changed judge noise requires recalibration and can change stopping order. Across 11,128 query--budget checks, exact releases earlier in 835 and later in none. The same queries recur across budgets; these checks share observations.

\paragraph{Judgment rule and fallback.}
Pairs are accepted when mean support across both prompt orientations reaches .50. The deployment judge returns probabilities over first-better (A), second-better (B), tie (T), and uncertain (U). For planned $i\to j$, support is $s_{ij}=\tfrac12(p_1(B)+p_1(T)+p_2(A)+p_2(T))$. Writing $\bar p_X=\tfrac12(p_1(X)+p_2(X))$, reverse support satisfies $s_{ij}+s_{ji}=1-\bar p_U+\bar p_T$. Rejection therefore need not imply reverse support above .50. For three-label A/B/T answers, as in the separate full-table diagnostic, $\bar p_U=0$ and that implication holds. The experimental policy retains only planned support edges. It can discard useful reverse information when reverse support passes the threshold; using such edges would change acquisition and calibration. Tie mass supports weak comparisons but does not assert equality. This rule was selected with the static arm on development calibration queries. Their combined selection and calibration do not satisfy the frozen-pipeline assumption. We apply the guarantee to confirmation, whose settings were frozen before results were opened, subject to exchangeability. Adaptive trajectories were recollected with the selected rule inside the loop: changing acceptance can change later questions. Replaying relaxed acceptance on trajectories collected with a strict rule failed a replay-validity check and is excluded.

The fallback allows at most 30 score-improving swaps or replacements, each checked against the no-evidence risk relative to the original incumbent, with numerical tolerance $10^{-9}$. Thus accepted local edits do not accumulate unchecked loss relative to an evolving reference. BT uses the same paid comparisons, including probability information from rejected edges, with a score prior; it has information beyond the accepted weak assertions used by the certificate.

\section{Minimum Sufficient Evidence}
\label{app:evidence}
For correct support edges $E^*_{\mathrm{sup}}$, choosing the smallest certifying subset is a hitting-set problem:
\begin{align}
 B^*_{\mathrm{ord}}=&\min_{z\in\{0,1\}^{|E^*_{\mathrm{sup}}|}}\sum_e z_e,\\
 &\sum_{\substack{e=(i,j)\in E^*_{\mathrm{sup}}\\i\in S,\ j\notin S}}z_e\ge1
 \quad\forall S:v(S)>\varepsilon.\nonumber
\end{align}
Every unsafe binary configuration must contradict at least one selected comparison. The optimum is zero when $R_0\le\varepsilon$, and infinite if complete support evidence cannot certify. Equation~\ref{eq:sandwich} implies $B^*_{\mathrm{ord}}\le B^*_{\mathrm{flow}}$.

A binary master covers unsafe sets found so far. The exact risk oracle adds a surviving unsafe witness until the master solution certifies. Each optimally solved master yields a lower bound, attained by the first certifying solution. Only optimal solver statuses count as an exact optimum. These are hindsight requirements based on correct evidence, excluding rejected or incorrect purchases.

Complete correct support evidence certifies 146 proposals with exact risk and 142 with flow among 183 development queries. Of the 169 actually within-budget proposals, arbitrary weak comparisons certify one further case. The remaining 22 still admit unsafe gains under the complete weak order. Among the 142 common support cases, 71 require fewer edges under exact risk and none require more. The exact requirement has median nine and maximum eleven, versus a flow maximum of twelve.

At $B=8$, greedy-flow acquisition certifies 6.6\% with flow and 10.4\% with exact risk; cover with exact risk reaches 19.7\%. At $B=32$, cover reaches the support-interface ceiling, $146/183=79.8\%$. Completion and greedy flow reach 78.7\% and 78.1\%, with mean budget-truncated costs 15.3 and 16.5, versus cover's 15.5. Completion can cost less than cover despite certifying fewer proposals.

\section{Expanded Empirical Results}
\label{app:results}
\begin{table}[t]
\centering\small\setlength{\tabcolsep}{5pt}
\begin{tabular}{lrrrr}
\toprule
Panel & $n$ & $r_P\le0$ & $0<r_P\le\varepsilon$ & $r_P>\varepsilon$\\
\midrule
Development & 183 & 160 & 9 & 14\\
\bottomrule
\end{tabular}
\caption{Signed proposal losses on the complete TREC-DL 2020--2022 development panel, including its 61 calibration and 122 test queries; $\varepsilon=.05$. Positive loss means a worse proposal. The non-positive group includes 152 improvements and eight ties.}
\label{tab:proposal-risk}
\end{table}

Table~\ref{tab:proposal-risk} separates tolerated degradation from exceeding the loss budget. The counts use the same full development proposals as the evidence-ceiling analysis. They characterize the proposal distribution, without conditioning on a gate's acceptance.

TREC-DL 2020--2022 provides 183 development queries with a 61/122 calibration/test split; TREC-DL 2023 provides 82 confirmation queries with a 27/55 split \citep{craswell2020trecdl,craswell2021trecdl,craswell2022trecdl}. TREC candidate pools contain judged documents. BEIR pools are unfiltered; unjudged candidates receive relevance zero. Gains are $2^{\max(\mathrm{rel},0)}-1$. We use $\alpha=.10$ and budgets $0,8,16,32$, supplemented by $4,12$. BEIR calibration is pooled across SciFact, NFCorpus, and FiQA. The second proposer shares confirmation queries and is not an independent query panel.

Three greedy-flow contrasts were registered before confirmation results were opened; none of their $B=32$ mean-gain tests reached corrected significance. Cover-trajectory results are secondary analyses. Reported intervals use paired bootstrap resampling; tests use Holm correction within each contrast family.

\begin{table*}[t]
\centering\small
\begin{tabular}{llrrrrr}
\toprule
Panel & $n$ & $B$ & \multicolumn{2}{c}{$\Delta$nDCG} & \multicolumn{2}{c}{comparisons} \\
\cmidrule(lr){4-5}\cmidrule(lr){6-7}
 & & & diff. & $p_{\mathrm{Holm}}$ & diff. & $p_{\mathrm{Holm}}$ \\
\midrule
Development & 122 & 12 & +0.0278 & 0.0015 & -0.43 & 0.0023 \\
 &  & 16 & +0.0550 & 0.0015 & -1.11 & 0.0015 \\
 &  & 32 & +0.0063 & 1.000 & -2.34 & 0.0015 \\
\addlinespace
Confirmation & 55 & 12 & +0.0223 & 0.124 & -0.53 & 0.0096 \\
 &  & 16 & +0.0211 & 0.422 & -1.15 & 0.0010 \\
 &  & 32 & +0.0000 & 1.000 & -1.98 & 0.0010 \\
\addlinespace
Second proposer & 55 & 12 & +0.0380 & 0.099 & -0.47 & 0.019 \\
 &  & 16 & +0.0337 & 0.187 & -1.22 & 0.0015 \\
 &  & 32 & +0.0047 & 1.000 & -2.31 & 0.0015 \\
\addlinespace
BEIR & 848 & 12 & +0.0090 & 0.0010 & -0.15 & 0.0010 \\
 &  & 16 & +0.0151 & 0.0010 & -0.35 & 0.0010 \\
 &  & 32 & +0.0074 & 0.0022 & -0.99 & 0.0010 \\
\bottomrule
\end{tabular}

\caption{Exact minus flow gates on shared cover trajectories. Positive gain differences and negative comparison differences favor exact. Tests are corrected within the archived contrast families; these are secondary analyses.}
\end{table*}
\begin{table*}[t]
\centering\small
\begin{tabular}{llrrrrrr}
\toprule
Corpus & Gate & $n$ & $\Delta$nDCG & accept & viol. & cmp & saved (CI) \\
\midrule
SciFact & flow & 199 & +0.0343 & 0.136 & 2/27 & 29.8 &  \\
 & exact &  & +0.0361 & 0.146 & 2/29 & 29.2 & +0.533 $[+0.312,\,+0.784]$ \\
\addlinespace
NFCorpus & flow & 216 & +0.0243 & 0.139 & 2/30 & 29.7 &  \\
 & exact &  & +0.0325 & 0.194 & 5/42 & 28.7 & +0.986 $[+0.639,\,+1.389]$ \\
\addlinespace
FiQA & flow & 433 & +0.0577 & 0.242 & 2/105 & 29.0 &  \\
 & exact &  & +0.0673 & 0.282 & 4/122 & 27.8 & +1.199 $[+0.945,\,+1.485]$ \\
\bottomrule
\end{tabular}

\caption{Per-corpus BEIR results at $B=32$ on shared cover trajectories. ``Viol.'' counts loss-budget exceedances among accepted proposals; the calibration guarantee is marginal over the pooled query distribution.}
\end{table*}
\begin{table*}[t]
\centering\small
\begin{tabular}{@{}lrrrr@{}}
\toprule
 & \multicolumn{2}{c}{Exact certificate} & \multicolumn{2}{c}{BT / reference} \\
\cmidrule(lr){2-3}\cmidrule(l){4-5}
Decision rule & $\Delta$nDCG & Viol./acc. & $\Delta$nDCG & Viol./acc. \\
\midrule
\multicolumn{5}{@{}l}{\textit{Development, $n=122$}} \\[1pt]
Counterexample cover & +0.1423 & 1/64 & +0.1493 & 4/78 \\
Greedy flow & +0.1363 & 1/60 & +0.1593 & 4/84 \\
Flow completion & +0.1399 & 1/62 & +0.1570 & 4/81 \\
Capacity static & +0.0162 & 1/12 & +0.1400 & 2/64 \\
\addlinespace[2pt]
Score-only predictor & -- & -- & +0.1091 & 5/56 \\
Always deploy & -- & -- & +0.1981 & 8/122 \\
\midrule
\multicolumn{5}{@{}l}{\textit{Confirmation, $n=55$}} \\[1pt]
Counterexample cover & +0.0964 & 0/24 & +0.1291 & 0/43 \\
Greedy flow & +0.0882 & 0/23 & +0.1385 & 0/45 \\
\addlinespace[2pt]
Score-only predictor & -- & -- & +0.0990 & 0/37 \\
Always deploy & -- & -- & +0.1542 & 1/55 \\
\bottomrule
\end{tabular}

\caption{Deployment at $B=32$. Gains average over all test queries; ``Viol./acc.'' counts budget exceedances among accepted proposals. BT uses the same paid comparisons, with probability information beyond accepted assertions. References in the last two rows per panel do not buy comparisons.}
\end{table*}

The development gain difference at $B=16$ under cover is .0550, CI $[.0324,.0805]$. Confirmation and second-proposer pointwise gain intervals exclude zero, but their Holm-adjusted tests are .422 and .187. Primary-budget cost savings persist across panels. Acquisition strongly affects this gain: development comparison differences at $B=32$ are $-2.34$ for cover, $-.98$ for greedy flow, and $-.20$ for completion; the completion test has adjusted $p=.334$.

Predictors release more confirmation proposals. Certificates have fewer observed development violations, but these samples do not establish better safety at a common marginal risk target. Always deploying gains .1981 on development and .1542 on confirmation, with violations on 8/122 and 1/55 queries. These empirical proportions do not provide a conformal guarantee. PRP-Allpair, which constructs a new ranking, uses 135.3 comparisons per development query and gains .2408; its calibrated wrapper deploys every query, with 4.1\% exceeding $\varepsilon$. Its comparison cost is about six times exact cover's 22.4. The fixed-proposal certificate does not dominate the overall quality--cost frontier.

\section{False Assertions and Decision Tolerance}
\label{app:errors}
A non-relevance assertion specifies $g_d=0$; let $Z\subseteq C$ collect these assertions. Count distinct directional comparisons and zero assertions as evidence items. For non-empty $S\subseteq C$, let
\begin{align}
\nu_{E,Z}(S)=&\#\{(i,j)\in E:i\in S,j\notin S\}\nonumber\\
&+\#\{d\in Z:d\in S\}.
\end{align}
With at most $k$ false accepted assertions, the exact full-corpus risk is
\begin{align}
\mathcal V_k&=\{v(S):S\in\mathcal S_k\},\nonumber\\
\mathcal S_k&=\{S: \varnothing\ne S\subseteq C,\ \nu_{E,Z}(S)\le k\},\nonumber\\
R^{(k)}_{\mathrm{mix}}(E,Z)&=\max\bigl(\{0\}\cup\mathcal V_k\bigr),
\end{align}
where the sets in the second family are non-empty subsets of $C$. Delete the assertions violated by an admissible gain vector. Every threshold set then satisfies the remaining assertions and violates at most $k$, so the proof of Theorem~\ref{thm:exact} applies. Conversely, each included binary set violates exactly $\nu_{E,Z}(S)$ items. Zero is attained by placing relevance only outside $C$, assuming the corpus contains such a document. For $k=0$ and $Z=\varnothing$, this reduces to the main theorem.

Define
\begin{equation}
\kapeps(E,Z)=\min_{S:v(S)>\varepsilon}\nu_{E,Z}(S),
\end{equation}
with an empty minimum equal to infinity. Then $R^{(k)}_{\mathrm{mix}}\le\varepsilon$ exactly when $\kapeps\ge k+1$. An actual error count $M<\kapeps$ suffices for a within-budget decision. Failure of this sufficient condition does not imply an actual loss violation. The robust self-checks verify duplicate idempotence, equivalence to deleting up to $k$ assertions, and this duality on randomized instances.

\paragraph{Full-table judge diagnostic.}
On a separate 183-query development panel, Qwen3-32B-AWQ answered every unordered pair in both orientations and one pointwise query per document (52,639 forwards). Prompts were selected on ten design queries. Strict acceptance requires agreement at .80; relaxed acceptance uses mean support at .50. This qrels-based full-table audit has no calibration/test split.

Under relaxed acceptance, 135 queries certify at $k=1$. Of these, 15 satisfy $M\le1$, another 15 satisfy $M<\kapeps$, 100 lack these sufficient guarantees but stay within the observed loss budget, and five exceed it. Across all 183 queries, $M<\kapeps$ covers 30 under strict acceptance and 34 under relaxed acceptance, including some $k=0$-only decisions; none exceeds the budget. Using this offline diagnostic to select decisions would require knowing the true error count.

\section{Experiment Materials}
\label{app:materials}
The code repository provides experiment materials and certificate implementations. Main-text deployment results and the full-table judge diagnostic are separate analyses. The deployment experiments use the calibration and test splits described in Appendix~\ref{app:protocol}; the full-table diagnostic uses all 183 development queries to measure assertion errors and decision tolerance. Its 52,639 forwards include both pair orientations and pointwise relevance judgments. Diagnostic acceptance rules are evaluated separately from the calibrated deployment gates.
\end{document}